\documentclass[12pt]{article}
\usepackage[margin=1in]{geometry}

\author{Paul Christiano \\ UC Berkeley \\ \texttt{paulfchristiano@eecs.berkeley.edu}}
\title{Collaborative Prediction with Expert Advice}

\usepackage{amsfonts}
\usepackage{amssymb}
\usepackage{amsmath}
\usepackage{latexsym}
\usepackage{mathtools}
\usepackage{enumerate}
\usepackage{colonequals}

\usepackage{amsthm}

\newtheorem{theorem}{Theorem}
\newtheorem*{theorem*}{Theorem}
\newtheorem{corollary}{Corollary}
\newtheorem*{corollary*}{Corollary}
\newtheorem{lemma}{Lemma}
\usepackage{algorithm2e}

\renewcommand{\epsilon}{\varepsilon}

\newcommand{\users}{\ensuremath{\mathcal{U}}}

\newcommand{\mwm}{\mathrm{MWM}}
\newcommand{\collabmwm}{\mathrm{MWM}_{\users}}
\newcommand{\Omego}[1]{\Omega\of{#1}}
\newcommand{\omego}[1]{\omega\of{#1}}
\newcommand{\experts}{\ensuremath{\mathcal{X}}}

\newcommand{\entropy}[1]{H_b\of{#1}}

\newcommand{\OPTT}{\mathrm{OPT}_{\leq T}}
\newcommand{\OPTmH}{\mathrm{OPT}_{\leq T}^{H,m}}
\newcommand{\abs}[1]{\left|#1\right|}
\newcommand{\OPTH}{\OPTT^H}

\newcommand{\logo}[1]{\log{#1}}

\renewcommand{\O}[1]{\mathcal{O}\of{#1}}
\renewcommand{\o}[1]{o\of{#1}}
\newcommand{\Ot}[1]{\widetilde{\mathcal{O}}\of{#1}}
\renewcommand{\o}[1]{o\of{#1}}

\newcommand{\ps}[1]{\left(#1\right)}
\DeclarePairedDelimiter{\pdelims}{(}{)}
\newcommand{\floor}[1]{\left\lfloor#1\right\rfloor}
\newcommand{\of}[1]{\pdelims*{#1}}

\renewcommand{\b}[1]{\left\{#1\right\}}

\newcommand{\lossH}{\ell_{\leq T}^H}

\newcommand{\theregret}{ \Ot{\sqrt{N \of{V_H \alpha \log \frac 1{\alpha} + V_{\users \backslash H} \ps{1 - \alpha}\log \frac 1{1 - \alpha}}}}}

\newcommand{\Let}[2]{#1$\leftarrow$#2\;}

\begin{document}

\maketitle

\begin{abstract}
Many practical learning systems aggregate data across many users,
while learning theory traditionally considers a single
learner who trusts all of their observations.
A case in point is the foundational learning problem of prediction with expert advice.
To date, there has been no theoretical study
of the general \emph{collaborative} version of prediction with expert advice,
in which many users face a similar problem and would like to share
their experiences in order to learn faster.
A key issue in this collaborative 
framework is robustness: generally algorithms that aggregate data are vulnerable to 
manipulation by even a small number of dishonest users.

We exhibit the first robust collaborative algorithm for prediction with expert advice.
When all users are honest and have similar tastes our algorithm
matches the performance of pooling data and using a traditional algorithm.
But our algorithm also guarantees that adding users
never significantly degrades performance,
even if the additional users behave adversarially.
We achieve strong guarantees even when the overwhelming
majority of users behave adversarially.
As a special case, our algorithm is extremely robust to
variation amongst the users.
\end{abstract}

\section{Introduction}
Modern machine learning systems often aggregate data from many users
to make a range of significant decisions,
from product recommendations that shape what we buy
to search rankings that shape what we read.
Sharing data facilitates rapid learning,
but leaves these systems vulnerable to manipulation by malicious users.
We consider a formal model of robust collaborative algorithms,
which offer performance guarantees even if many users behave maliciously.

Users in our model face the traditional problem
of using advice from $M$ experts
to make a sequence of decisions or predictions.
Different users could solve their prediction problems independently:
by using standard techniques,
we could ensure that each user makes about $\O{\logo{M}}$ suboptimal predictions
before converging to the performance of the best single expert.
However, if some experts make good predictions for many users,
then those users should be able to share their data in order to learn faster.
Rather than having \emph{each user} make $\O{\logo{M}}$
bad predictions, \emph{all of the users together}
could make only $\O{\logo{M}}$ bad predictions.
If there are $\omego{\logo{M}}$ users, the typical user would make $\o{1}$ bad predictions.

Ideally we would have a collaborative algorithm
which exploits shared structure when it exists,
but which is robust to differences amongst users.
As an important extreme case,
we would like algorithms which achieve meaningful bounds even when
some users behave adversarially.
For example, many ``users'' of a search engine
may be manipulators, trying to promote their clients' web pages.
A naive application of traditional learning algorithms to the collaborative
setting would be vulnerable to manipulation
even when the number of manipulators is $\o{1}$.

To date, there has been no theoretical study
of the simplest collaborative version of prediction with expert advice.
In particular, no existing algorithms achieve fast convergence
while remaining robust to either inhomogeneous preferences or a small fraction
of malicious users.
We propose a new algorithm for this setting which meets both goals.
Our algorithm provides very strong guarantees
even when the vast majority of users are dishonest manipulators.

The starting point for our approach
is a novel reduction to the problem of learning from specialists,
experts who sometimes decline to offer advice \cite{specialists}.
This reduction requires
an expert for every subset of the users, yielding an exponential time algorithm.
Our key contribution is to ``reverse'' the learning problem,
having each expert learn which subset of the users it should offer advice to.
This idea leads to an efficient algorithm that enjoys the same guarantees.

\subsection{Our model}

We fix a set of users $\users$ and a set of experts $\experts$.
In each round $t = 1, 2, \ldots$ a single user $u_t$ must pick an expert $x_t \in \experts$
(their choice may be randomized).
After choosing, $u_t$ observes a vector of losses $\ell_t : \experts \rightarrow [-1, 1]$,
and receives the loss $\ell_t\of{x_t}$.
Finally, $u_t$ may post the vector $\ell_t$ to a public bulletin board
(a dishonest user may instead post an arbitrary vector).
The contents of the bulletin board are visible to other users in future rounds.

If a single expert $x_H$ predicts well for all of the users in some set $H$,
then the users in $H$ ought to be able to share their data
in order to identify $x_H$ more quickly.
The difficulty is that the set $H$ is not known in advance,
and so we cannot simply aggregate data from all users in $H$ without including
data from users who are malicious or for whom $x_H$ does not predict well.

To make the goal formal, let $H$ be an arbitrary set of users
who honestly report their payoffs.
Define the loss of $\lossH$ as the total loss in all rounds
involving a user in $H$:
\[\lossH = \sum_{t \leq T : u_t \in H} \ell_t\of{x_t}.\]

We compare this loss to the best performance that the users in $H$
could have achieved,
if they had chosen a single fixed expert:
\[\OPTH = \min_{x \in \experts} \sum_{t \leq T: u_t \in H} \ell_t\of{x}. \]
We are interested in bounding the \emph{regret} $\lossH - \OPTH$.

We write $N = \abs{\users}$ and $M = \abs{\experts}$.
$T$ is the total number of rounds, which we do not assume is known in advance.
We write $\Ot{\cdot}$ to hide additive terms of $\O{\sqrt{T \log\log T}}$.
These terms do not affect the asymptotics unless $T > 2^{M}$
which is not a parameter regime we are interested in.

We define and analyze an algorithm $\collabmwm$.
A more precise regret bound is given in Theorem~\ref{fullstatement} in Section~\ref{alltogether},
but the following simple corollary captures the basic behavior:
\begin{corollary}\label{simplestatement}
Let $H$ be any set of users. 
Then $\collabmwm$ satisfies
\[ \lossH \leq \OPTH + \O{\sqrt{T \ps{\log{M} + N}}}. \]
If $\abs{H} = \alpha N$ and $u_t \in H$ in an $\alpha$ fraction of rounds,
then we have the tighter bound:
\[ \lossH \leq \OPTH + \Ot{\sqrt{\alpha T \ps{\log M + N \entropy{\alpha}}}} \]
where $\entropy{\alpha} = \alpha \log \frac 1{\alpha} +  (1 - \alpha) \log \frac 1{1 - \alpha}$
is the binary entropy.
\end{corollary}
(The full statement removes the assumption that $u_t \in H$ for an $\alpha$ 
fraction of rounds, and provides a significantly tighter bound that depends
on the actual sequence of payoffs.)

To understand this bound,
consider the regret incurred by the typical user in $H$
over their first $k = T/N$ decisions.

If the users made decisions independently,
the per-user regret would be $\O{\sqrt{k \logo{M}}}$.
If the set $H$ was given in advance so that the users could perfectly share their information,
then the per-user regret would be $\O{\sqrt{k \frac{\logo{M}}{\abs{H}}}}$---this is the best
that we can hope to achieve by any collaborative algorithm.

For $\alpha = \Theta\of{1}$,
we bound the per-user regret by $\O{\sqrt{k\ps{ \frac {\logo{M}}{\abs{H}} + 1}}}$.

For $\alpha > 1/2$, our bound is
$\O{\sqrt{k\ps{ \frac {\logo{M}}{\abs{H}} + (1-\alpha) \log \frac 1{1 - \alpha}}}}$,
which converges to $\O{\sqrt{k \frac{\logo{M}}{\abs{H}}}}$
as $\alpha\rightarrow 1$.

For $\alpha < 1/2$,
our bound is $\O{\sqrt{k\ps{\frac{\logo{M}}{\abs{H}} + \log{\frac 1{\alpha}}}}}$.
We can see this is optimal in the case
where the users are divided into $\frac 1{\alpha}$ independent clusters:
even if the clustering of other users is given,
each user would incur regret $\O{\sqrt{k \logo{\frac 1{\alpha}}}}$
to learn which of the clusters they belong to.

These regret bounds imply
a bound against a stronger benchmark,
in which we divide $H$ up into $m$ groups $H_1, \ldots, H_m$
and choose the optimal $x_i \in \users$ for each group:
\[\OPTmH = \min_{H_1 \cup \cdots \cup  H_m = H} \sum_i \min_{x_i \in \users} \sum_{t \leq T : u_t \in H_i} \ell_t\of{x_i}\]
The algorithm $\collabmwm$ satisfies:
\begin{corollary}\label{optmh}
For any set of users $H$ and any $m > 1$:
\[ \lossH \leq \OPTmH + \Ot{\sqrt{T \ps{m \logo{M} + N \logo{m}}}} \]
\end{corollary}
That is, the group collectively pays the regret required to solve $m$ parallel expert problems,
and each user pays the regret required to solve an experts problem with $m$ experts.

\subsection{Related work}

\textbf{Competitive collaborative learning}
\cite{ccl} addresses a collaborative version of the multi-armed bandit problem.
In their approach, each user learns either a good arm or a \emph{single}
other user to whom they delegate their decision
(that user may in turn delegate further).
In contrast, achieving our regret bounds requires sharing data across
\emph{all} sufficiently similar users.
This makes the problems conceptually distinct,
and they require completely different techniques.

\textbf{Collaborative filtering}
has been studied at length
and is probably the best understood setting for collaborative learning; see \cite{survey} for an overview.
A wide range of theoretical models for this problem
have been studied (\cite{mefiltering}, \cite{tellme}, \cite{matrix-prediction},  \cite{strangers}, \cite{recommenders}).

Collaborative filtering is closely related to the special case of collaborative prediction with expert advice
in which experts correspond to sets of ``good'' resources.

This is an important special case, but it does not capture the general behavior of prediction with expert advice.
The single-user version of collaborative filtering is typically trivial---try each resource
and discover which are good.

In contrast,
the single-user version of prediction with expert advice is a foundational problem
in learning theory.
So understanding how to generalize prediction with expert advice to the collaborative setting
is a natural and important step towards understanding collaborative learning in general.
Existing techniques for collaborative filtering cannot be applied to general prediction with expert advice,
and new techniques are needed.

\textbf{Adversarial learning.} Another literature
deals with learning problems in which an adversary has some influence
over the training or testing data \cite{adversaries}.
Our model of robust collaboration can be viewed within this framework,
as an attack model in which an adversary controls the data associated with some users.
The unique characteristic of our model is that we only care about the performance of our model
in rounds involving uncorrupted users; in our view
this is a very natural model of an important class of attacks,
and it allows us to obtain extremely strong regret bounds.

\section{Our algorithm}\label{peasection}

\subsection{Background: single-user prediction with expert advice}

As a subroutine, we will need to use a traditional algorithm
$\mwm$ based on multiplicative updates \cite{mwm}.
We will need to use a variant which tolerates different learning rates and initial
weights for different experts.
This variant provides three functions:
\newcommand{\init}[1]{\mathrm{INIT}\of{#1}}
\newcommand{\update}[2]{\mathrm{UPDATE}\of{#1, #2}}
\begin{itemize}
\item $\init{\experts, w, \epsilon}$, where $w$ and $\epsilon$ are positive vectors indexed by $x \in \experts$
with $\sum_x w\of{x} = 1$.
This outputs a new ``instance'' $A$, with initial weights $w$ and learning rates $\epsilon$.
The other routines are called with an instance as an argument.
If the weights are missing we assume they are uniform.
We may specify a single learning rate $\epsilon$ for all of the experts.
\item $\update{A}{\ell_t}$, where $\ell_t$ is a vector indexed by $x \in \experts$
with entries in $[-1, 1]$.
This updates the weights of $A$ based on the loss vector $\ell_t$, and outputs the new instance.
\item $A(x)$, where $x \in \experts$, outputs the current weight of expert $x$.
These weights are guaranteed to be non-negative and sum to $1$.
\end{itemize}
We write $\ell_t\of{A} = \sum_{x \in \experts} A\of{x} \ell_t\of{x}$.
More generally, if $p_t \in \Delta\of{\experts}$ is any probability distribution,
write $\ell_t\of{p_t} = \sum_{x \in \experts} p_t\of{x} \ell_t\of{x}$.

$\mwm$ satisfies the following performance guarantee:
\begin{lemma}\label{mwm}
For any $\experts, w, \epsilon, T$, any sequence of loss vectors $\ell_t$, and any $x^* \in \experts$:
\[ \sum_t \ell_t\of{A_t}  \leq\sum_t \ell_t\of{x^*} +
\epsilon\of{x^*} \sum_t \ps{\ell_t\of{A_t} - \ell_t\of{x^*}}^2 + \O{\frac{\logo{ \frac 1{w\of{x^*}}}}{\epsilon\of{x^*}}}, \]
where $A_1 = \init{\experts, w, \epsilon}$ and $A_{t+1} = \update{A_{t}}{\ell_{t}}$.
\end{lemma}
\begin{proof}
$\mwm$ internally maintains a set of weights $w\of{x}$ that sum to $1$.
The actual outputs $A(x)$ are proportional to $w\of{x} \epsilon\of{x}$.
The weights $w$ are then updated according to the rule
\[ w\of{x} \gets w\of{x} \ps{1 + \epsilon\of{x} \ps{\ell\of{x} - \ell\of{A}}}. \]
It is easy to verify that this rule exactly preserves the total weight.
The claimed regret bound then follows immediately
from the traditional analysis of multiplicative weight updates,
as in \cite{mwm}.
\end{proof}

\subsection{Basic algorithm}\label{basicalg}
In this section we describe our basic algorithm and prove a bound on its regret.
In the following sections we define and analyze two improvements on the basic algorithm
which achieve significantly stronger regret bounds.

We will now assume that the number of rounds $T$ is known---because all of our regret
bounds are $\Omego{\sqrt{T}}$, this assumption can easily be removed by a standard doubling trick.

In each round $t$,
each expert $x \in \experts$ decides whether it wants to offer advice
to the user $u_t$.
We then aggregate their advice using $\mwm$.
We need to cope with the fact that experts only offer
advice in a subset of the rounds;
for this we use a standard trick \cite{specialists}
to hold fixed each expert's (normalized) weight
during rounds where it does not offer advice.

Each expert $x$ itself uses an online learning algorithm
to decide when it should be willing to offer advice.
Expert $x$'s goal is to offer advice only when doing so
will increase its own weight.
This is roughly equivalent to offering advice only when doing so
will decrease the ``excess loss'' of expert $x$,
the difference between expert $x$'s loss
and the overall loss of our algorithm.
(Though the two are not equivalent, as discussed in Section~\ref{minvariance},
and our full algorithm must pay attention to the difference.)

Now suppose that $\lossH$ is significantly less than $\OPTH$.
This implies that the optimal expert $x_H$
could significantly increase its own weight
by choosing to offer advice precisely in rounds where $u_t \in H$.
Since $x_H$ offers advice in a nearly optimal set of rounds,
we conclude that the weight of $x_H$ must grow nearly as fast
as if it had offered advice only to users in $H$.
This leads to a bound on how much $\lossH$ can exceed
$\OPTH$.

In the basic version of our algorithm,
the expert decides whether to offer advice to user $u_t$
based only on their previous experiences with $u_t$.

%

\newcommand{\atx}{A_t\of{x}}

\begin{algorithm}\label{weakpea}
    \caption{Collaborative prediction with expert advice [simplified algorithm]}
    \Let{$A_1$}{$\init{\experts, \epsilon = \sqrt{\logo{M} / T}}$}
    \For{$x \in \experts$, $u \in \users$}{
        \Let{$B^{xu}$}{$\init{\b{0, 1}, \epsilon = \sqrt{N/T}}$}
    }
    \For{$t = 1, 2, \ldots$}{
        Observe $u_t \in \users$\;
        \For{$x \in \experts$}{
            \Let{$z_t^x$}{$B^{xu_t}\of{1}$}
            \Let{$w_t^x$}{$z_t^x \atx$}
        }
        \Let{$W_t$}{$\sum_x w_t^x$}
        Play $p_t\of{x} = w_t^x / W_t \in \Delta\of{\experts}$\;
        Observe $\ell_t : \experts \rightarrow [-1, 1]$\;
        \For{$x \in \experts$}{
            \Let{$\ell_t^A\of{x}$}{$z_t^x \ell_t\of{x} + (1 - z_t^x) \ell_t\of{p_t}$}
            \Let{$\ell_t^{B^{xu_t}}\of{1}$}{$\ell_t\of{x}$}
            \Let{$\ell_t^{B^{xu_t}}\of{0}$}{$\ell_t\of{p_t}$}
            \Let{$B^{xu_t}$}{$\update{B^{xu_t}}{\ell_t^{B^{xu_t}}}$}
        }
        \Let{$A_{t+1}$}{$\update{A_t}{\ell_t^A}$}
    }
\end{algorithm}

Our first lemma
shows that the excess loss of expert $x$ in the rounds
where it opts to make a prediction
is at most the excess loss of expert $x$ in rounds involving a user $u_t \in H$.
\begin{lemma}\label{reverselemma}
For every $x \in \experts$:
\[\sum_{t \leq T} z_t^x \ps{\ell_t\of{x} - \ell_t\of{p_t}} \leq \sum_{t \leq T : u_t \in H} \ps{\ell_t\of{x} - \ell_t\of{p_t}} + \O{\sqrt{T N}}\]
\end{lemma}
\begin{proof}

We apply Lemma~\ref{mwm} to each $\mwm$ instance $B^{xu}$,
and sum the resulting inequalities.
Write $h_u = 1$ if $u \in H$, and $0$ otherwise, and let $T_u = \abs{\b{t : u_t = u}}$.
\begin{align*}
\sum_{t \leq T} z_t^x \ps{\ell_t\of{x} - \ell_t\of{p_t}}
&= \sum_{u \in \users} \sum_{t \leq T : u_t = u} z_t^x \ps{\ell_t\of{x} - \ell_t\of{p_t}} \\
&\leq \sum_{u \in \users}  \ps{
\O{\sqrt{T/N} + T_u \sqrt{N/T}} + 
\sum_{t \leq T : u_t = u} h_u \ps{\ell_t\of{x} - \ell_t\of{p_t}}
} \\
&= \O{\sqrt{NT}} + \sum_{t \leq T : u_t \in H} \ps{\ell_t\of{x} - \ell_t\of{p_t}}
\end{align*}
as desired.
\end{proof}

\newcommand{\ra}{ \O{\sqrt{T \logo{M}}}}

Our second lemma shows that excess loss of an expert, in the rounds where it makes a prediction,
cannot be too large.
\begin{lemma}\label{specialistlemma}
For any $x \in \experts$,
\[\sum_{t \leq T} z_t^x \ps{\ell_t\of{p_t} - \ell_t\of{x}}  \leq \ra \]
\end{lemma}
\begin{proof}
First, we observe that $\ell^A_t\of{A_t} = \ell_t\of{p_t}$:
\begin{align*}
\ell^A_t\of{A_t} 
&= \sum_x \atx \ps{z_t^x \ell_t\of{x} + (1 - z_t^x) \ell_t\of{p_t}} \\
&= \sum_x\atx z_t^x \ell_t\of{x} + \ell_t\of{p_t} \sum_x \atx - \ell_t\of{p_t} \sum_x \atx z_t^x \\
&= \sum_x w_t^x \ell_t\of{x}  + \ell_t\of{p_t} - \ell_t\of{p_t} \sum_x w_t^x \\
&= W_t \ell_t\of{p_t} + \ell_t\of{p_t} - W_t \ell_t\of{p_t} \\
&= \ell_t\of{p_t}.
\end{align*}
So we can apply the regret bound for $A$, and obtain:
\begin{align*}
\sum_{t \leq T} \ell_t\of{p_t}
&= \sum_{t \leq T} \ell^A_t\of{A_t} \\
&\leq \sum_{t \leq T} \ell^A_t\of{x} + \ra \\
&= \sum_{t \leq T} \ps{z_t^x \ell_t\of{x} + (1 - z_t^x) \ell_t\of{p_t}} + \ra \\
&= \sum_{t \leq T} \ell_t\of{p_t} + \sum_{t \leq T} z_t^x \ps{\ell_t\of{x} - \ell_t\of{p_t}} + \ra \\
\sum_{t \leq T} z_t^x \ps{\ell_t\of{p_t} - \ell_t\of{x}} &\leq \ra 
\end{align*}
\end{proof}

\begin{theorem}
For each $x$ and $H$,
algorithm~\ref{weakpea} satisfies
\[\lossH = \sum_{t \leq T: u_t \in H }{\ell_t\of{p_t}}
\leq \sum_{t \leq T : u_t \in H} {\ell_t\of{x}}
+ \O{\sqrt{T\ps{\log{M} + N}}}\]
\end{theorem}
\begin{proof}
Applying Lemma~\ref{reverselemma} and then Lemma~\ref{specialistlemma}:
\begin{align*}
\sum_{t \leq T : u_t \in H} \ps{\ell_t\of{p_t} - \ell_t\of{x}}
&\leq \sum_{t \leq T} z_t^x \ps{\ell_t\of{p_t} - \ell_t\of{x}} + \O{\sqrt{T N}} \\
&\leq \ra + \O{\sqrt{TN}},
\end{align*}
as desired.
\end{proof}

\subsection{Improving the algorithm: minimizing variance}\label{minvariance}

The regret bound in the previous section depends on $\sqrt{T \logo{M}}$.
If $H$ is very small then this bound is problematic, since $T$
may be much larger than the number of rounds $T_H$ involving an honest user.

Suppose that $H$ is small and that $x_H$ is the expert who is optimal for
users in $H$.
Intuitively, if users outside of $H$ don't like $x_H$'s recommendations,
then $x_H$ should only make recommendations to users in $H$,
and so we should end up with a regret term that depends on $T_H$.
On the other hand, if users outside of $H$ do like $x_H$'s recommendations,
then that should be even more helpful for quickly identifying $x_H$.

So what can go wrong?
Suppose that the users outside of $x_H$ are \emph{indifferent}
to $x_H$'s recommendations---half of the time $\ell_t\of{x_H}$
is one less than $\ell_t\of{p_t}$,
and half of the time $\ell_t\of{x_H}$ is one more.
In this case, $x_H$ will continue to happily make recommendations
to users outside of $H$.

But now consider what happens to $x_H$'s weight if it does make a recommendation.
With probability $1/2$ it is multiplied by $(1 + \epsilon)$,
and with probability $1/2$ it is multiplied by $(1 - \epsilon)$.
The net effect of these two steps is to multiply $x_H$'s weight by
$(1 + \epsilon)(1 - \epsilon) = 1 - \epsilon^2$.
In general this ``volatility drag'' is $\epsilon^2 \ps{\ell_t \of{p_t} - \ell_t \of{x_H}}^2$,
and it occurs in every round where $x_H$ makes a recommendation.

To solve this problem,
we incorporate this drag into the expert's optimization problem.
That is, we adjust the losses
$\ell_t^{B^{xu}}(1)$ by adding the quadratic penalty $\epsilon \ps{\ell_t\of{p_t} - \ell_t\of{x}}^2$.
This corresponds to having the expert maximize their expected \emph{log weight} rather than their expected weight.

After making this change the analysis of the previous section can be
adapted to yield an improved regret bound that depends on $\sqrt{T_H \logo{M}}$.
The only additional difficulty is that we need to adjust the learning rate based on $T_H$,
which we don't know.
We overcome this difficulty by introducing a whole family of parallel experts with exponentially
distributed learning rates.
This leads to a regret of $\Ot{\sqrt{T_H \logo{M}}}$,
where the $\Ot{\cdot}$ hides an additive $\O{\sqrt{T \log\log T}}$.

In fact we can replace $T_H$ in the bound with the \emph{variance}, as in \cite{squint},
\[V_H = \sum_{t : u_t \in H} \ps{\ell_t\of{x} - \ell_t\of{p_t}}^2,\]
achieving a bound that mirrors Lemma~\ref{mwm}.

This improvement is included in our final algorithm in Section~\ref{alltogether}.

\subsection{Improving the algorithm: learning the base rate}

In our basic algorithm, the experts treat each user as a separate learning problem.
We can improve the algorithm by having the experts learn what fraction of the users are honest,
rather than implicitly expecting half of all users to be honest.

\newcommand{\mwms}{\mathrm{MWM}_{\theta}}

We introduce a new learning algorithm $\mwms$ for solving a simultaneous
prediction with expert advice problem for each user $u \in \users$.
$\mwms$ implements a similar interface to $\mwm$:

\newcommand{\inits}[1]{\mathrm{INIT}_{\theta}\of{#1}}
\newcommand{\updates}[3]{\mathrm{UPDATE}_{\theta}\of{#1, #2, #3}}
\begin{itemize}
\item $\inits{\users}$, where $\users$ is a set.
This outputs a new ``instance'' $A$.
The other routines are called with an instance as an argument.
\item $\updates{A}{u_t}{\ell_t}$, where $u_t \in \users$  and $\ell_t \in [-1, 1]$.
This updates the weights of $A$ based on the loss $\ell_t$ incurred by the user $u_t$, and outputs the new instance.
\item $A(u_t)$, where $u_t \in \users$, outputs a probability in $[0, 1]$.
\end{itemize}

Roughly speaking, $\mwms$ works by instantiating one expert for each parameter $\theta$
in $[0, 1]$.
That expert treats each user $u_t$ independently,
but has a ``prior'' probability of $\theta$ for each user.
$\mwms$ then competes with the best of these experts.

In Appendix~\ref{mwmsappendix}, we define $\mwms$
and prove the following result:
\begin{theorem}\label{mwmsbound}
For any $\users$ and $H \subset \users$,
any sequence of users $u_t \in \users$, and any sequence of losses $\ell_t \in [-1, 1]$,
we have:
\[ \sum_{t \leq T} \ell_t A_t\of{u_t}
\leq
\sum_{t \leq T : u_t \in H} \ell_t +
\theregret
\]
where $A_1 = \inits{\users}$, $A_{t+1} = \updates{A_t}{u_{t}}{\ell_{t}}$,
$V_H = \sum_{t \leq t : u_t \in H} \ps{\ell_t\of{x} - \ell_t\of{p_t}}^2$ and $\alpha = \abs{H} / N$.
\end{theorem}

%

With $\mwms$ in hand we can further improve Algorithm~\ref{weakpea}.
Rather than having each expert instantiate a separate instance $B^{xu}$ of $\mwm$
for each user $u$,
we have them instantiate a single instance $B^x$ of $\mwms$.
The analysis of the improved algorithm is then identical
to the analysis of Algorithm~\ref{weakpea},
except that the conclusion of Lemma~\ref{reverselemma} is strengthened appropriately.
The result is precisely the strengthened conclusion in Theorem~\ref{fullstatement}.
This improvement is incorporated into the full algorithm in the next section.

\subsection{Putting it all together}\label{alltogether}

In this section we update Algorithm~\ref{weakpea} to incorporate
the improvements described in the last two sections.
The result is Algorithm~\ref{strongpea}.

\newcommand{\epsV}{\epsilon_V}
\newcommand{\vs}{\mathcal{V}}
\newcommand{\diffsquared}{\ps{\ell_t\of{x} - \ell_t\of{p_t}}^2}
\begin{algorithm}\label{strongpea}
    \caption{Collaborative prediction with expert advice [full algorithm]}
    \Let{$\vs$}{$\b{1, 2, 4, \ldots, 2^{\floor{\log_2 T}}}$}
    \Let{$\epsV$}{$\min\b{1, \sqrt{\ps{\log{M} + \log\log{T}}/V}}$}
    \Let{$A_1$}{$\init{\experts \times \vs, \epsilon\of{x, V} = \epsV}$}
    \For{$x \in \experts, V \in \vs $}{
        \Let{$B^{x,V}_1$}{$\inits{\users}$}
    }
    \For{$t = 1, 2, \ldots$}{
        Observe $u_t \in \users$\;
        \For{$x \in \experts, V \in \vs$}{
            \Let{$z_t^{x, V}$}{$B^{x, V}_t\of{u_t}$}
            \Let{$w_t^{x, V}$}{$z_b^x \atx$}
        }
        \Let{$W_t$}{$\sum_{x, V} w_t^{x, V}$}
        Play $p_t\of{x} = \sum_V w_t^{x, V} / W_t \in \Delta\of{\experts}$\;
        Observe $\ell_t : \experts \rightarrow [-1, 1]$\;
        \For{$x \in \experts, V \in \vs$}{
            \Let{$\ell_t^A\of{x, V}$}{$z_t^{x, V} \ell_t\of{x} + \ps{1 - z_t^{x, V}} \ell_t\of{p_t}$}
            \Let{$\ell_t^{B^{x, V}}$}{$\ell_t\of{x} - \ell_t\of{p_t} + \epsV \diffsquared$}
            \Let{$B^{x, V}_{t+1}$}{$\updates{B^{x, V}_t}{u_t}{\ell_t^{B^{x, V}}}$}
        }
        \Let{$A_{t+1}$}{$\update{A_t}{\ell_t^A}$}
    }
\end{algorithm}
\begin{lemma}\label{reverselemmaplus}
For every $x \in \experts, V \in \vs, H \subset \users$:
\begin{align*}
\sum_{t \leq T} z_t^{x, V}
\ps{\ell_t\of{x} - \ell_t\of{p_t} + \epsV \diffsquared}
\leq & \sum_{t \leq T : u_t \in H} \ps{\ell_t\of{x} - \ell_t\of{p_t}}
+ \O{V_H \epsV }\\
&+\theregret
\end{align*}
where $V_H = \sum_{t \leq t : u_t \in H} \ps{\ell_t\of{x} - \ell_t\of{p_t}}^2$ and $\alpha = \abs{H} / N$.
\end{lemma}
\begin{proof}
We apply Lemma~\ref{mwms} directly to the instance $B^{x, V}$.
Note that $V_H \epsV$
is precisely the total loss caused by the penalty term $\epsV \ps{\ell_t\of{x} - \ell_t\of{p_t}}^2$
in all rounds $t$ with $u_t \in H$.

After adding this quadratic penalty term the payoffs $\ell_t^{B^{x, V}}$
are no longer in $[-1, 1]$, but they are still $\O{1}$.
This does not affect the asymptotics in the regret bound.
\end{proof}

\newcommand{\zt}{z_t^{x, V}}
\begin{lemma}\label{specialistlemmaplus}
For any $x \in \experts, V \in \vs$,
\[\sum_{t \leq T} \zt\ps{\ell_t\of{p_t} - \ell_t\of{x}}
\leq \epsV \sum_{t \leq T} \zt \diffsquared + \O{\frac{\log{M} + \log\log{T}}{\epsV}}\]
\end{lemma}
\begin{proof}
As before, we observe that $\ell^A_t\of{A_t} = \ell_t\of{p_t}$:
\renewcommand{\atx}{A_t\of{x, V}}
\begin{align*}
\ell^A_t\of{A_t} 
&= \sum_{x, V} \atx \ps{\zt \ell_t\of{x} + (1 - \zt) \ell_t\of{p_t}} \\
&= \sum_{x, V} \atx \zt \ell_t\of{x} + \ell_t\of{p_t} \sum_{x, V} \atx - \ell_t\of{p_t} \sum_{x, V} \atx \zt \\
&= \sum_{x, V} w_t^{x, V} \ell_t\of{x}  + \ell_t\of{p_t} - \ell_t\of{p_t} \sum_{x, V} w_t^{x, V} \\
&= W_t \ell_t\of{p_t} + \ell_t\of{p_t} - W_t \ell_t\of{p_t} \\
&= \ell_t\of{p_t}.
\end{align*} 
So we can apply the regret bound for $A$, and obtain:
\newcommand{\zdiff}{\zt \ell_t\of{x} + \ps{1 - \zt} \ell_t\of{p_t}}
\begin{align*}
\sum_{t \leq T} \ell_t\of{p_t}
=& \sum_{t \leq T} \ell^A_t\of{A_t} \\
\leq & \sum_{t \leq T} \ell^A_t\of{x, V} + \epsV \sum_{t \leq T} \ps{\ell^A_t\of{x, V} - \ell^A_t\of{p_t}}^2 + \O{\frac{\log{M} + \log\log{T}}{\epsV}}\\
=& \sum_{t \leq T} \ps{\zdiff}  \\
&+ \epsV \sum_{t \leq T} \ps{\zdiff - \ell_t\of{p_t}}^2  \\
&+ \O{\frac{\log{M} + \log\log{T}}{\epsV}}\\
=& \sum_{t \leq T} \ell_t\of{p_t} + \sum_{t \leq T} \zt \ps{\ell_t\of{x} - \ell_t\of{z_t}}
+ \epsV \sum_{t \leq T} \ps{\zt \ps{\ell_t\of{x} - \ell_t\of{p_t}}}^2 \\
&+ \O{\frac{\log{M} + \log\log{T}}{\epsV}}\\
\leq& \sum_{t \leq T} \ell_t\of{p_t} + \sum_{t \leq T} \zt \ps{\ell_t\of{x} - \ell_t\of{z_t}}
+ \epsV \sum_{t \leq T} \zt \ps{\ell_t\of{x} - \ell_t\of{p_t}}^2 \\
&+ \O{\frac{\log{M} + \log\log{T}}{\epsV}}\\
\sum_{t \leq T} \zt \ps{\ell_t\of{p_t} - \ell_t\of{x}} &\leq \epsV \sum_{t \leq T} \zt \diffsquared + \O{\frac{\log{M} + \log\log{T}}{\epsV}},
\end{align*}
as desired.
\end{proof}

\begin{theorem}\label{fullstatement}
For each $x$ and $H$,
algorithm~\ref{strongpea} satisfies
\[ \sum_{t \leq T: u_t \in H }{\ell_t\of{p_t}}
\leq \sum_{t \leq T : u_t \in H} {\ell_t\of{x}}
+ \Ot{\sqrt{V_H \log{M} + N\of{V_H \alpha \log{\frac 1{\alpha}} + V_{\users \backslash H} (1 - \alpha) \log{\frac 1{1 - \alpha}}}}}\]
Where $V_H = \sum_{t \leq T : u_t \in H}{\diffsquared}$ and $\alpha = \abs{H}/N$.
\end{theorem}
\begin{proof}
Note that $V_H \leq T$.
Thus there exists some $V \in \vs$ with $V \leq V_H \leq 2V$.

Applying Lemma~\ref{reverselemmaplus} with the pair $x, V$ we have:
\begin{align*}
\sum_{t \leq T : u_t \in H} \ps{\ell_t\of{p_t} - \ell_t\of{x}}
&\leq
\sum_{t \leq T} z_t^{x, V}
\ps{\ell_t\of{p_t} - \ell_t\of{x} - \epsV \diffsquared}
+ \O{\sqrt{V_H \logo{M}}}\\
&+\theregret
\end{align*}
Applying Lemma~\ref{specialistlemmaplus}:
\begin{align*}
\sum_{t \leq T} z_t^{x, V}
\ps{\ell_t\of{p_t} - \ell_t\of{x} - \epsV \diffsquared}
&\leq \O{\frac{\log{M} + \log\log{T}}{\epsV}}
&\leq \Ot{\sqrt{V_H \log{M}}}
\end{align*}
Combining the two inequalities gives the desired result.
\end{proof}
Corollary~\ref{simplestatement} follows immediately
from the observation that $V_H \leq T_H$,
and the inequality $\alpha^2 \log{\frac 1{\alpha}} + \of{1 - \alpha}^2 \log{\frac 1{1 - \alpha}} \leq 2 \alpha H_b\of{\alpha}$.

\subsection{Competing with $\OPTmH$}
We now prove Corollary~\ref{optmh}:
\begin{corollary*}[Restatement of Corollary~\ref{optmh}]
For any set of users $H$ and any $m > 1$:
\[ \lossH \leq \OPTmH + \Ot{\sqrt{T \ps{m \logo{M} + N \logo{m}}}} \]
\end{corollary*}
\begin{proof}
We apply Corollary~\ref{simplestatement} to each of the sets $H_i$, and sum the resulting inequalities.
Let $\alpha_i = \abs{H_i} / N$, and $\alpha = \abs{H} / N$.
\begin{align*}
\sum_{t : u_t \in H_i} \ell_t\of{p_t}
&\leq \sum_{t : u_t \in H_i} \ell_t\of{x_i} + \Ot{\sqrt{\alpha_i T \ps{\logo{M} + H_b\of{\alpha_i} N}}} \\
\sum_{t : u_t \in H} \ell_t\of{p_t}
&\leq \OPTmH + \Ot{\sum_i \sqrt{\alpha_i T \ps{\logo{M} + H_b\of{\alpha_i} N}}} \\
&\leq \OPTmH + \Ot{\sqrt{T N} + \sum_i \sqrt{\alpha_i T \ps{\logo{M} + \alpha_i N \log \alpha_i}}}
\end{align*}
We have $\sum \alpha_i = \alpha$,
and this regret bound is a concave function of $\alpha_i$.
So by Jensen's inequality 
we can replace $\alpha_i$ with $\alpha/m$:
\begin{align*}
\sum_{t : u_t \in H_i} \ell_t\of{p_t}
&\leq \OPTmH + \Ot{\sqrt{T N} + \sum_i \sqrt{\alpha_i T \ps{\logo{M} + \alpha_i N \log \frac 1{\alpha_i}}}} \\
&\leq \OPTmH + \Ot{\sqrt{T N} + m \sqrt{\alpha T / m \ps{\logo{M} + N / m \log \frac 1{\alpha} + N / m \log m}}} \\
&\leq \OPTmH + \Ot{\sqrt{T N} + \sqrt{\alpha T \ps{m \logo{M} + N \log \frac 1{\alpha} + N \log m}}} \\
&\leq \OPTmH + \Ot{\sqrt{T \ps{m \logo{M} + N \log m}}}
\end{align*}
as desired.

We can only apply Corollary~\ref{simplestatement} when about $\Theta\of{\alpha_i T}$ rounds involve
the users in $H_i$, for each $i$.
In general, we can make the same argument by applying Theorem~\ref{fullstatement},
and applying convexity again to assume $V_{H_i} = V_H / m$.
\end{proof}

\section{Open questions}

The robust collaborative learning framework
provides a general transformation from single-user learning
problem to robust collaborative learning problems.
We have answered a few fundamental questions,
but we leave many more open.

\begin{itemize}
\item \textbf{Parallel expert problems.} Suppose the same set of users participate
in many online services $1, 2, \ldots, k$.
The same users may behave honestly,
and the same groups of users may tend to share tastes,
across many different online services.
We would like to be able to amortize the additional regret over all of these services,
rather than running a separate collaborative learning algorithm for each of them.
This corresponds to an experts problem with a simple combinatorial structure:
an ``expert'' corresponds to a choice of expert 
in each of the $k$ underlying problems.
We can apply our results in this setting, but the runtime is exponential in $k$
since we must explicitly represent each expert.
\cite{mefiltering} essentially solves
the special case where the number of experts in each problem is $2$.
But the general problem remains open,
and their regret bounds are suboptimal.
\item \textbf{Online convex optimization} Online convex optimization
is an extremely general learning problem.
Our algorithm can be adapted to online convex optimization,
but the resulting algorithm is intractable.
Understanding how to generalize online convex optimization to the collaborative setting
is a natural next step towards a general theory of collaborative learning.
\item \textbf{Bandit feedback.} Our algorithms all require full feedback.
It seems likely that they can be extended to the contextual bandits setting,
which would be important for many practical applications.
Without some additional stochastic assumptions,
we expect that the regret will have to be $\Omego{\sqrt{TAN}}$,
where $A$ is the number of available actions.
Even this result would greatly improve the practical applicability of our algorithm.
It is not obvious how to generalize our results even when $A = 2$,
without obtaining regret that depends on $T^{2/3}$.
\item \textbf{Exploiting side information about users.} Our regret bounds
depend on a quantity like $N H_b{\alpha}$,
representing the prior probability of $H$ under a natural distribution.
In realistic settings, there is significant side information about users
that may help us guess which users are honest,
and help us predict which users will have similar preferences.
For example, users who are friends with each other
may be especially likely to have common tastes (and to either
both be honest or neither be honest).
Incorporating this kind of side information is non-trivial,
but could potentially lead to much stronger bounds.
\item \textbf{Memory requirements.} Our algorithm for prediction with expert advice
requires maintaining one weight for each (expert, user) pair.
When the number of users and experts is large, this may be infeasible.
A more efficient algorithm might only require $\O{\abs{\users} + \abs{\experts}}$
storage rather than $\O{\abs{\users} * \abs{\experts}}$ storage.

\end{itemize}

\bibliographystyle{acm}
\bibliography{collab}

\appendix

\section{Defining $\mwms$}\label{mwmsappendix}

\newcommand{\ns}{\mathcal{N}}
\renewcommand{\epsV}[2]{\epsilon_{N_{#1}, V_{#1}}}
\newcommand{\indices}{N_D, N_H, V_D, V_H}
\newcommand{\loopindices}{N_D, N_H \in \ns, V_D, V_H \in \vs}
\begin{algorithm}\label{mwms}
    \caption{$\mwms$}
    \Let{$\vs$}{$\b{1, 2, 4, \ldots, 2^{\floor{\log_2 T}}}$}
    \Let{$\ns$}{$\b{1, 2, 4, \ldots, 2^{\floor{\log_2 N}}}$}
    \Let{$\Theta$}{$\ns^2 \times \vs^2$}
    \Let{$\epsV{S}{S}$}{$\sqrt{\ps{N_S \log \frac {N}{N_S} + \log\log T}/V_S}$}
    \Let{$A_1$}{$\init{\Theta, \epsilon = \sqrt{\ps{\log\log T + \log\log N} / T}}$}
    \For{$\loopindices, u \in \users$}{
        \Let{$B_1^{\indices, u}$}{$\init{\b{0, 1}, w(0) \propto N_D, w(1) \propto N_H, \epsilon(0) = \epsV{D}{D}, \epsilon(1) = \epsV{H}{H}}$}
    }
    \For{$t = 1, 2, \ldots$}{
        Observe $u_t \in \users$\;
        Play $p_t = \sum_{\theta \in \Theta} A_t\of{\theta} B_t^{\theta, u_t}\of{1}$\;
        Observe $\ell_t \in [-1, 1]$\;
        \For{$\theta \in \Theta$}{
            \Let{$\ell_t^A\of{\theta}$}{$\ell_t B_t^{\theta, u_t}(1)$}
            \Let{$B_{t+1}^{\theta, u_t}$}{$\update{B_t^{\theta, u_t}}{\ell_t}$}
        }
        \Let{$A_{t+1}$}{$\update{A_t}{\ell_T^A}$}
    }
\end{algorithm}

$\mwms$ is defined in Figure~\ref{mwms}.
$C \gets \inits{\users}$ runs the code before the loop over $t$.
$C\of{u_t}$ returns $p_t$.
$\updates{C}{u_t}{\ell_t}$ advances the loop over $t$.

\begin{theorem*}[Restatement of Theorem~\ref{mwmsbound}]
\end{theorem*}
For any $\users$ and $H \subset \users$,
any sequence of users $u_t \in \users$, and any sequence of losses $\ell_t \in [-1, 1]$,
we have:
\[ \sum_{t \leq T} \ell_t A_t\of{u_t}
\leq
\sum_{t \leq T : u_t \in H} \ell_t +
\theregret
\]
where $A_1 = \inits{\users}$, $A_{t+1} = \updates{A_t}{u_{t}}{\ell_{t}}$,
where $V_H = \sum_{t \leq t : u_t \in H} \ps{\ell_t\of{x} - \ell_t\of{p_t}}^2$ and $\alpha = \abs{H} / N$.
\begin{proof}

For every $\theta \in \Theta$
we have
\begin{align*}
\sum_t p_t \ell_t 
&= \sum_t \sum_{\theta} A_t \of{\theta} B_t^{\theta, u_t}\of{1} \ell_t \\
&= \sum_t \sum_{\theta} A_t \of{\theta} \ell_t^A\of{\theta} \\
&= \sum_t \ell_t^A \of{A_t} \\
&\leq \sum_t \ell_t B_t^{\theta, u_t}\of{1} + \O{\sqrt{T \of{\log\log N + \log \log T}}} \\
&= \sum_t \ell_t B_t^{\theta, u_t}\of{1} + \Ot{\sqrt{T \log \log N}}
\end{align*}
Let $V_u = \sum_{t \leq T : u_t = u}\of{\ell_t h_u - \ell_t B^{\theta, u}}^2$.

Now note that there exists a $\theta$ for which $\indices$
are all within a factor of two of their intended values, i.e.
\begin{align*}
N_H &\approx  \abs{H} \\
N_D &\approx  \abs{\users \backslash H} \\
V_H &\approx \sum_{u \in H} V_u \\
V_D &\approx \sum_{u \not \in H} V_u.
\end{align*}

We can apply Lemma~\ref{mwm} to the corresponding instances $B^{\theta, u}$
and sum the resulting inequalities across $u$:

\begin{align*}
\sum_{t : u_t = u} B^{\theta, u}_t\of{1} \ell_t 
\leq& \sum_{t : u_t = u} \ell_t + V_u \epsV{H}{H} + \O{\frac{\log{N/N_H}}{\epsV{H}{H}}} \\
\sum_{t : u_t = u} B^{\theta, u}_t\of{1} \ell_t 
\leq& V_u \epsV{D}{D} + \O{\frac{\log{N/N_D}}{\epsV{D}{D}}} \\
\sum_{t} B^{\theta, u_t}_t \of{1} \ell_t
\leq& \sum_{t : u_t \in H} \ell_t + \sum_{u \in H} V_u \epsV{H}{H} + \sum_{u \not \in H} V_u \epsV{D}{D} \\
&+ \O{\frac{\abs{H} \log{N/N_H}}{\epsV{H}{H}} + \frac{\abs{\users \backslash H} \log{N / N_D}}{\epsV{D}{D}}} \\
\leq& \sum_{t : u_t \in H} \ell_t + \O{V_H \epsV{H}{H} + V_D \epsV{D}{D}} \\
&+ \O{\frac{N_H \log{N/N_H}}{\epsV{H}{H}} + \frac{N_D \log{N / N_D}}{\epsV{D}{D}}} \\
\leq& \sum_{t : u_t \in H} \ell_t + \O{\sqrt{V_H N_H \log N/N_H}}
+ \O{\sqrt{V_D N_D \log N/N_D}} \\
=& \sum_{t : u_t \in H} \ell_t  + \theregret
\end{align*}

Combining these two inequalities, we obtain
\begin{align*}
\sum_t p_t \ell_t &\leq \sum_t \ell_t B_t^{\theta, u}\of{1} + \Ot{\sqrt{T \log \log N}} \\
&\leq \sum_{t : u_t \in H} \ell_t + \theregret + \Ot{\sqrt{T \log \log N}} \\
&= \sum_{t : u_t \in H} \ell_t + \theregret,
\end{align*}
as desired.
The last equality holds whenever $H \neq \users$ and $H \neq \emptyset$ because $\log \log N \leq N H_b\of{\alpha}$ for
any $1/N < \alpha < 1 - 1/N$.


Coping with the case $H = \users$ or $H = \emptyset$ actually requires a very slight adjustment to our algorithm:
we include in $\Theta$ a new pair of values $\theta = (0, 0, N, 0), (0, 0, 0, N)$,
and have $A_1$ assign these values an initial weight of $1/3$ and learning rate of $T^{-1/2}$.
This does not affect the asymptotics of our regret bound,
but ensures that we have regret $\Ot{0}$ whenever $H = \users$ or $H = \emptyset$.
\end{proof}

\end{document}